\documentclass[letterpaper]{article}

\usepackage{xcolor}

\usepackage{uai2020new}
\usepackage[margin=1in]{geometry}
\usepackage{times,url,hyperref}

\usepackage{amsmath,amssymb,amsthm}
\usepackage{algorithm,algorithmic}
\usepackage[square,numbers]{natbib}

\usepackage{graphicx}
\usepackage{subcaption}
\newtheorem{theorem}{Theorem}
\newtheorem{lemma}[theorem]{Lemma}

\newtheorem{corollary}[theorem]{Corollary}

\newtheorem{definition}[theorem]{Definition}

\newcommand{\cO}{\mathcal{O}}

\newcommand{\cD}{\mathcal{D}}

\newcommand{\cX}{\mathcal{X}}
\newcommand{\E}{\mathbb{E}}

\title{Randomized Exploration for Non-Stationary Stochastic Linear Bandits}

\author{ {\bf Baekjin Kim}\thanks{Please note that this is a corrected version of the paper originally published in UAI 2020. }\\
Department of Statistics \\
University of Michigan\\
Ann Arbor, MI 48109 \\
\And
 {\bf Ambuj Tewari}\\
Department of Statistics \\
University of Michigan\\
Ann Arbor, MI 48109 
}

\begin{document}


\maketitle
\begin{abstract}
We investigate two perturbation approaches to overcome conservatism that optimism based algorithms chronically suffer from in practice. The first approach replaces optimism with a simple randomization when using confidence sets. The second one adds random perturbations to its current estimate before maximizing the expected reward. For non-stationary linear bandits, where each action is associated with a $d$-dimensional feature and the unknown parameter is time-varying with total variation $B_T$, we propose two randomized algorithms, Discounted Randomized LinUCB (D-RandLinUCB) and Discounted Linear Thompson Sampling (D-LinTS) via the two perturbation approaches. We highlight the statistical optimality versus computational efficiency trade-off between them in that the former asymptotically achieves the optimal dynamic regret $\tilde{O}(d^{7/8} B_T^{1/4}T^{3/4})$, but the latter is oracle-efficient with an extra logarithmic factor in the number of arms compared to minimax-optimal dynamic regret. In a simulation study, both algorithms show outstanding performance in tackling conservatism issue that Discounted LinUCB struggles with. 
\end{abstract}

\section{INTRODUCTION}
A multi-armed bandit is the simplest model of decision making that involves the exploration versus exploitation trade-off \citep{lai1985asymptotically}. Linear bandits are an extension of multi-armed bandits where the reward has linear structure with a finite-dimensional feature associated with each arm \citep{abe2003reinforcement, dani2008stochastic}. Two standard exploration strategies in stochastic linear bandits are Upper Confidence Bound algorithm (LinUCB) \citep{abbasi2011improved} and Linear Thomson Sampling (LinTS) \citep{agrawal2013thompson}. The former relies on optimism in face of uncertainty and is a deterministic algorithm built upon the construction of a high-probability confidence ellipsoid for the unknown parameter vector. The latter is a Bayesian solution that maximizes the expected rewards according to a parameter sampled from the posterior distribution. \citet{chapelle2011empirical} showed that Linear Thompson Sampling empirically performs better and is more robust to corrupted or delayed feedback than LinUCB. From a theoretical perspective, it enjoys a regret bound that is a factor of $\sqrt{d}$ worse than minimax-optimal regret bound $\tilde{\Theta}(d\sqrt{T})$ that LinUCB enjoys. However, the minimax optimality of optimism comes at a cost: implementing UCB type algorithms can lead to NP-hard optimization problems even for convex action sets \citep{agrawal2019recent}.

Random perturbation methods were originally proposed in the 1950s by \citet{hannan1957approximation} in the full information setting where losses of all actions are observed.  \citet{kalai2005efficient} showed Hannan's perturbation approach leads to efficient algorithms by making repeated calls to an offline optimization oracle. They also gave a new name to this family of randomized algorithms: Follow the Perturbed Leader (FTPL). Recent works \citep{abernethy2014online, abernethy2015fighting, kim2019optimality} have studied the relationship between FTPL and Follow the Regularized Leader (FTRL) algorithms and also investigated whether FTPL algorithms achieve minimax-optimal regret in full and partial information settings.

\citet{abeille2017linear} viewed Linear Thompson Sampling as a perturbation based algorithm, characterized a family of perturbations whose regrets can be analyzed, and raised an open problem to find a minimax-optimal perturbation. In addition to its significant role in smartly balancing exploration with exploitation, a perturbation based approach to linear bandits also reduces the problem to one call to the offline optimization oracle in each round. Recent works \citep{kveton2019perturbed, kveton2019randomized} have proposed randomized algorithms that use perturbation as a means to achieve oracle-efficient computation as well as better theoretical guarantee than LinTS, but there is still a gap between their regret bounds and the lower bound of $\Omega(d\sqrt{T})$. This gap is logarithmic in the number of actions which can introduce extra dependence on dimension for large action spaces.

A new randomized exploration scheme was proposed in the recent work of \citet{vaswani2019old}. In contrast to Hannan's perturbation approach that injects perturbation directly into an estimate, they replace optimism with random perturbation when using confidence sets for action selection in optimism based algorithms. This approach can be broadly applied to multi-armed bandit and structured bandit problems, and the resulting algorithms are theoretically optimal and empirically perform well since overall conservatism of optimism based algorithms can be tackled by randomizing the confidence level.

Linear bandit problems were originally motivated by applications such as online ad placement with features extracted from the ads and website users. However, users' preferences often evolve with time, which leads to interest in the non-stationary variant of linear bandits. Accordingly, adaptive algorithms that accommodate time-variation of environments have been studied in a rich line of works in both multi-armed bandit \citep{besbes2014stochastic} and linear bandit. With prior information of total variation budget, SW-LinUCB \citep{cheung2019hedging}, D-LinUCB \citep{russac2019weighted}, and Restart-LinUCB \citep{zhao2020simple} were constructed on the basis of the optimism in face of uncertainty principle via sliding window, exponential discounting weights and restarting, respectively. Recently, \citet{zhao2021non} discovered a technical mistake shared in three prior works and presented a fix which deteriorates their dynamic regret bounds from $\tilde{O}(d^{2/3} B_T^{1/3}T^{2/3})$ to $\tilde{O}(d^{7/8} B_T^{1/4}T^{3/4})$. In addition, \citet{luo2017efficient} and \citet{chen2019new} studied fully adaptive and oracle-efficient algorithms assuming access to an optimization oracle when total variation is unknown for the learner. It is still open problem to design a practically simple, oracle-efficient and statistically optimal algorithm for non-stationary linear bandits.

\subsection{CONTRIBUTION}
In Section \ref{sec:stationary}, we explicate, in the simpler stationary setting, the role of two perturbation approaches in overcoming conservatism that UCB-type algorithms chronically suffer from in practice. In one approach, we replace optimism with a simple randomization when using confidence sets. In the other, we add random perturbations to the current estimate before maximizing the expected reward. These two approaches result in Randomized LinUCB and Gaussian Linear Thompson Sampling for stationary linear bandits. We highlight the statistical optimality versus oracle efficiency trade-off between them.

In Section \ref{sec:non}, we study the non-stationary environment and present two randomized algorithms with exponential discounting weights, Discounted Randomized LinUCB (D-RandLinUCB) and Discounted Linear Thompson Sampling (D-LinTS) to gracefully adjust to the time-variation in the true parameter. We explain the trade-off between statistical optimality and oracle efficiency in that the former asymptotically achieves the optimal dynamic regret $\tilde{O}(d^{7/8} B_T^{1/4}T^{3/4})$, but the latter enjoys computational efficiency due to sole reliance on an offline optimization oracle for large or infinite action set. However it incurs an extra $(\log K)^{3/8}$ gap in its dynamic regret bound, where $K$ is the number of actions.

In Section \ref{sec:exp}, we run multiple simulation studies based on Criteo live traffic data \citep{diemert2017attribution} to evaluate the empirical performances of D-RandLinUCB and D-LinTS. We observe that when high dimension and a large set of actions are considered, the two show outstanding performance in tackling conservatism issue that the non-randomized D-LinUCB struggles with.

\section{WARM-UP: STATIONARY STOCHASTIC LINEAR BANDIT}\label{sec:stationary}
\subsection{PRELIMINARIES}
In stationary stochastic linear bandit, a learner chooses an action $X_t$ from a given action set $\cX_t \subset \mathbb{R}^d$ in every round $t$, and he subsequently observes a reward $Y_t = \langle X_t, \theta^{\star} \rangle +\eta_t$ where $\theta^{\star} \in \mathbb{R}^d$ is an unknown parameter and $\eta_t$ is a conditionally 1-subGaussian random variable. For simplicity, assume that $\|\theta^{\star}\|_2 \le 1$ and, for all $x\in \cX_t$, $\|x\|_2 \le 1$, and thus $|\langle x, \theta^{\star}\rangle|_2 \le 1$.

As a measure of evaluating a learner, the regret is defined as the difference between rewards the learner would have received had it played the best in hindsight, and the rewards actually received. Therefore, minimizing the regret is equivalent to maximizing the expected cumulative reward. Denote the best action in a round $t$ as $x_t^{\star} = \arg\max_{x\in \cX_t} \langle x, \theta^{\star}\rangle$ and the expected regret as $E[R(T)]  = \E\big[\sum_{t=1}^T [\langle x_t^{\star},\theta^{\star}\rangle - \langle X_t,\theta^{\star}\rangle]\big]$.

To learn about unknown parameter $\theta^{\star}$ from history up to time $t-1$, $\mathcal{H}_{t-1} = \{ (X_l,Y_l)_{1\le l\le t-1} \}$, algorithms rely on  $l^2$-regularized least-squares estimate of $\theta^{\star}$, $\hat{\theta}^{ls}_t$, and confidence ellipsoid centered from $\hat{\theta}^{ls}_t$.  We define
$\hat{\theta}^{ls}_t  = V_{t,\lambda}^{-1} \sum_{l=1}^{t-1} X_l Y_l$, where  $V_{t,\lambda} = \lambda I_d + \sum_{l=1}^{t-1} X_l X_l^T$ and $\lambda$ is a positive regularization parameter.
\begin{table*}[t]
\caption{Algorithms in stationary stochastic linear bandits}
\label{table}
\begin{center}
\begin{small}
\begin{tabular}{lcccr}
\multicolumn{1}{c}{\bf ALGORITHM} & \multicolumn{1}{c}{\bf REGRET BOUND} & \multicolumn{1}{c}{\bf RANDOMNESS} & \multicolumn{1}{c}{\bf ORACLE ACCESS}\\
\\[-.7em]
\hline
\\[-.7em]
LinUCB \citep{abbasi2011improved}  & $\tilde{\mathcal{O}}(d\sqrt{T})$  & No & No \\
LinTS \citep{agrawal2013thompson}  & $\tilde{\mathcal{O}}(d^{3/2}\sqrt{T})$ & Yes & Yes \\
{\bf Gaussian LinTS \citep{kveton2019randomized}}  & $\tilde{\mathcal{O}}(d\sqrt{T\log K})$ & Yes & Yes \\
LinPHE \citep{kveton2019perturbed}    & $\tilde{\mathcal{O}}(d\sqrt{T\log K})$ & Yes & Yes \\
RandLinUCB \citep{vaswani2019old}& $\tilde{\mathcal{O}}(d\sqrt{T})$       & Yes & No \\
\\[-.7em]
\end{tabular}
\end{small}
\end{center}
\vskip -0.1in
\end{table*}

\subsection{RANDOMIZED EXPLORATION}\label{sec:rand_exp}

The standard solutions in stationary stochastic linear bandit are optimism based algorithm (LinUCB, \citet{abbasi2011improved}) and Linear Thompson Sampling (LinTS, \citet{agrawal2013thompson}). While the former obtains the theoretically optimal regret bound $\tilde{\cO}(d\sqrt{T})$ matched to lower bound $\Omega(d\sqrt{T})$, the latter empirically performs better in spite of its regret bound $\sqrt{d}$ worse than LinUCB \citep{chapelle2011empirical}. In finite-arm setting, the regret bound of Gaussian Linear Thompson Sampling (Gaussian-LinTS) is improved by $\sqrt{(\log K)/ d}$ as a special case of Follow-the-Perturbed-Leader-GLM (FPL-GLM, \citet{kveton2019randomized}). Also, a series of randomized algorithms for linear bandit were proposed in recent works: Linear Perturbed History Exploration (LinPHE, \citet{kveton2019perturbed}) and Randomized Linear UCB (RandLinUCB, \citet{vaswani2019old}). They are categorized in terms of regret bounds, randomness, and oracle access in Table \ref{table}, where we denote $K = \max_{t \in [T]} |\cX_t|$ in finite-arm setting.

There are two families of randomized algorithms according to the way perturbations are used. The first algorithm family is designed to choose an action by maximizing the expected rewards after adding the random perturbation to estimates. Gaussian-LinTS, LinPHE, and FPL-GLM are in this family. But they are limited in that their regret bounds, $\tilde{\cO} (d \sqrt{T \log K})$, depend on the number of arms, and lead to  $\tilde{\cO}(d^{3/2}\sqrt{T})$ regret bounds when the action set is infinite. The other family including RandLinUCB is constructed by replacing the optimism with simple randomization when choosing a confidence level to handle the chronic issue that UCB-type algorithms are too conservative. This randomized version of LinUCB matches optimal regret bounds of LinUCB as well as the empirical performance of LinTS.

{\bf Oracle point of view :}
We assume that the learner has access to an algorithm that returns a near-optimal solution to the offline problem, called an {\em offline optimization oracle}. It returns the optimal action that maximizes the expected reward from a given action space $\cX \subset
\mathbb{R}^d$ when a parameter $\theta \in \mathbb{R}^d$ is given as input.
\begin{definition}[Offline Optimization Oracle]
There exists an algorithm, $\mathcal{A.M.O.}$, which when given a pair of action space $\cX \subset
\mathbb{R}^d$, and a parameter $\theta \in \mathbb{R}^d$, computes
    $\mathcal{A.M.O.}(\cX, \theta) = 
    \arg\max_{x \in \cX} \;\langle x, \theta \rangle$.
\end{definition}
Both the non-randomized LinUCB and RandLinUCB are required to compute spectral norms of all actions $\|x\|_{V_{t,\lambda}^{-1}}$ in every round so that they cannot be efficiently implemented with an infinite set of arms. The main advantage of the algorithms in the first family such as Gaussian-LinTS, LinPHE, and FPL-GLM is that they rely on an offline optimization oracle in every round $t$ so that the optimal action can be efficiently obtained within polynomial times from large or even infinite action set. 

{\bf Improved regret bound of Gaussian LinTS :}
In FTL-GLM, it is required to generate perturbations and save $d$-dimensional feature vectors $\{ X_l \}_{l=1}^{t-1}$ in order to obtain perturbed estimate $\tilde{\theta}_t$ in every round $t$, which causes computation burden and memory issue for storage. However, once perturbations are Gaussian in the linear model, adding univariate Gaussian perturbations to historical rewards is the same as perturbing the estimate $\hat{\theta}_t$ by a multivariate Gaussian perturbation because of its linear invariance property, and the resulting algorithm is approximately equivalent to Gaussian Linear Thompson Sampling \citep{agrawal2013thompson} as follows.
\\[-1.5em]
\begin{align*}
    \tilde{\theta}_t & = \hat{\theta}_t + V_{t,\lambda}^{-1} \sum_{l=1}^{t-1} X_l Z_l^{(t)},\; Z^{(t)}_{l}\sim \mathcal{N}(0, a^2)\\
    & \approx \hat{\theta}_t + V_{t,\lambda}^{-1/2} Z^{(t)},\; Z^{(t)} \sim \mathcal{N}(0, a^2 I_d)\\
    & \; : \text{\bf Gaussian-LinTS}.
\end{align*}
It naturally implies the regret bound of Gaussian-LinTS is improved by $\sqrt{(\log K)/d}$ with finite action sets \citep{kveton2019randomized}.

{\bf Equivalence between Gaussian LinTS and RandLinUCB :} 
Another perspective of Gaussian-LinTS algorithm is that it is equivalent to RandLinUCB with {\em decoupled} perturbations across arms due to linearly invariant property of Gaussian random variables:
\begin{align*}
  \langle x, \tilde{\theta}_t \rangle & = \langle x, \hat{\theta}_t \rangle + x^T V_{t,\lambda}^{-1/2} Z^{(t)}, \; Z^{(t)} \sim \mathcal{N}(0, a^2 I_d)\\
  &= \langle x, \hat{\theta}_t \rangle + Z_{t,x}  \|x\|_{V_{t,\lambda}^{-1}}, \; Z_{t,x} \sim N(0,a^2) \\
  &\;  : \footnotesize{\text{\bf {\em Decoupled} RandLinUCB}}.  
\end{align*}
If perturbations are coupled, we compute the perturbed expected rewards of all actions using randomly chosen confidence level $Z_t\sim N(0,a^2)$ instead of $Z_{t,x}$. In the decoupled RandLinUCB where each arm has its own random confidence level, more variations are generated so that its regret bound have extra logarithmic gap that depends on the number of decoupled actions. 
In other words, the standard ({\em coupled}) RandLinUCB enjoys minimax-optimal regret bound due to coupled perturbations. However, there is a cost to its theoretical optimality: it cannot just rely on an offline optimization oracle and thus loses computational efficiency. We thus have a trade-off between efficiency and optimality described in two design principles of perturbation based algorithms.

\section{NON-STATIONARY STOCHASTIC LINEAR BANDIT}\label{sec:non}

\begin{table*}[t]
\caption{Algorithms in non-stationary stochastic linear bandits}
\label{table2}
\begin{center}
\begin{small}
\begin{tabular}{lcccr}
\multicolumn{1}{c}{\bf ALGORITHM} & \multicolumn{1}{c}{\bf REGRET BOUND} & \multicolumn{1}{c}{\bf RANDOMNESS} & \multicolumn{1}{c}{\bf ORACLE ACCESS}\\
\\[-.7em]
\hline
\\[-.7em]
 D-LinUCB \citep{russac2019weighted} & 
$\mathcal{O}(d^{\frac{7}{8}}B_T^{\frac{1}{4}}T^{\frac{3}{4}})$  
 & No & No \\
SW-LinUCB \citep{cheung2019hedging}  & 
$\mathcal{O}(d^{\frac{7}{8}}B_T^{\frac{1}{4}}T^{\frac{3}{4}})$
 & No & No \\
 Restart-LinUCB \citep{zhao2020simple}  & 
$\mathcal{O}(d^{\frac{7}{8}}B_T^{\frac{1}{4}}T^{\frac{3}{4}})$
 & No & No \\
{\bf D-RandLinUCB [Algorithm \ref{alg:alg2}] } & 
$\mathcal{O}(d^{\frac{7}{8}}B_T^{\frac{1}{4}}T^{\frac{3}{4}})$
& Yes & No \\
{\bf D-LinTS [Algorithm \ref{alg:alg3}]}  & 
$\mathcal{O}(d^{\frac{7}{8}}(\log K)^{\frac{3}{8}}B_T^{\frac{1}{4}}T^{\frac{3}{4}})$
 & Yes & Yes \\
\end{tabular}
\end{small}
\end{center}
\vskip -0.1in
\end{table*}

\subsection{PRELIMINARIES} 
In each round $t \in [T]$, an action set $\cX_t \in \mathbb{R}^d$ is given to the learner and it has to choose an action $X_t \in \cX_t$. Then, the reward $Y_t = \langle X_t, \theta_t^{\star} \rangle + \eta_t$ is observed to the learner where $\theta_t^{\star} \in \mathbb{R}^d$ is an unknown time-varying parameter and $\eta_t$ is a conditionally 1-subGaussian random variable. The non-stationary assumption allows unknown parameter $\theta_t^{\star}$ to be time-variant within total variation budget $B_T = \sum_{t=1}^{T-1} \|\theta_{t}^{\star} - \theta_{t+1}^{\star}\|_2$. It is a nice way of quantifying time-variations of $\theta_{t}^{\star}$ in that it covers both slowly-changing and abruptly-changing environments. For simplicity, assume $\|\theta_t^{\star}\|_2 \le 1$, for all $x\in \cX_t$, $\|x\|_2 \le 1$, and thus $|\langle x, \theta_t^{\star}\rangle|_2 \le 1$. 

In a similar way to stationary setting, denote the best action in a round $t$ as $x_t^{\star} = \arg\max_{x\in \cX_t} \langle x, \theta_t^{\star}\rangle$ and denote the expected dynamic regret as $E[R(T)] = \E\big[\sum_{t=1}^T [\langle x_t^{\star},\theta_t^{\star}\rangle - \langle X_t, \theta_t^{\star}\rangle]\big]$ where $X_t$ is chosen action at time $t$. The goal of the learner is to minimize the expected dynamic regret.

In a stationary stochastic environment where the reward has a linear structure, linear upper confidence bound algorithm (LinUCB) follows a principle of optimism in the face of uncertainty (OFU). Under this OFU principle, three recent works of \citet{cheung2019hedging, russac2019weighted, zhao2020simple} proposed sliding window linear UCB (SW-LinUCB), discounted linear UCB (D-LinUCB), and restarting linear UCB (Restart-LinUCB) which are non-stationary variants of LinUCB to adapt to time-variation of $\theta_t^{\star}$. First two algorithms rely on weighted least-squares estimators with equal weights only given to recent $w$ observations where $w$ is length of a sliding-window, and exponentially discounting weights, respectively. The last algorithm proceeds in epochs, and is periodically restarted to be resilient to the drift of underlying parameter $\theta_t$.

Three non-randomized algorithms based on three different approaches are known to achieve the dynamic regret bounds $\tilde{O} (d^{7/8} B_T^{1/4} T^{3/4})$ using Bandit-over-Bandit (BOB) mechanism \citep{cheung2019hedging} without the prior information on $B_T$, but share inefficiency of implementation with LinUCB \citep{abbasi2011improved} in that the computation of spectral norms of all actions are required. Furthermore, they are built upon the construction of a high-probability confidence ellipsoid for the unknown parameter, and thus they are deterministic and their confidence ellipsoids become too wide when high dimensional features are available. In this section, randomization exploration algorithms, discounted randomized LinUCB (D-RandLinUCB) and discounted linear Thompson sampling (D-LinTS), are proposed to handle computational inefficiency and conservatism that both optimism-based algorithms suffer from. The dynamic regret bound, randomness, and oracle access of algorithms are reported in Table \ref{table2}.

\subsection{WEIGHTED LEAST-SQUARES ESTIMATOR}
First, we study the weighted least-squares estimator with discounting factor $0<\gamma <1$. In the round $t$, the weighted least-squares estimator is obtained in a closed form,
$\hat{\theta}^{wls}_t =W_{t,\lambda}^{-1}  \sum_{s=1}^{t-1} \gamma^{-l} X_l Y_l$
where $W_{t,\lambda} = \sum_{l=1}^{t-1}\gamma^{-l} X_l X_l^T + \lambda \gamma^{-(t-1)} I_d$. Additionally, we define $\tilde{W}_{t,\lambda} = \sum_{l=1}^{t-1}\gamma^{-2l} X_l X_l^T + \lambda \gamma^{-2(t-1)} I_d$. This form is closely connected with the covariance matrix of $\hat{\theta}^{wls}_t$. For simplicity, we denote $V_t = W_{t,\lambda} \tilde{W}_{t,\lambda}^{-1} W_{t,\lambda}$.

\begin{lemma}[Weighted Least-Sqaures Confidence Ellipsoid, Theorem 1 \citep{russac2019weighted}]\label{lemma:wls}
Assume the stationary setting where $\theta_t^{\star} = \theta^{\star}$. For any $\delta >0$, 
\begin{align*}
P \big( \forall t\ge 1,\| \hat{\theta}^{wls}_t - \theta^{\star} \|_{W_{t,\lambda} \tilde{W}_{t,\lambda}^{-1} W_{t,\lambda}} \le \beta_t \big) \ge 1-\delta
\end{align*}
where $\beta_t =  \sqrt{\lambda}+ \sqrt{2 \log (1/\delta) + d \log ( 1+ \frac{(1- \gamma^{2t})}{\lambda d (1-\gamma^2)})}$.
\end{lemma}

While Lemma \ref{lemma:wls} states that the confidence ellipsoid $\mathcal{C}_t = \{ \theta \in \mathbb{R}^d : \| \theta - \theta^{wls}_t \|_{W_{t,\lambda} \tilde{W}_{t,\lambda}^{-1} W_{t,\lambda}} \le \beta_t \}$ contains true parameter $\theta_t^{\star}$ with high probability in stationary setting, the true parameter $\theta^{\star}_t$ is not necessarily inside the confidence ellipsoid $\mathcal{C}_t$ in the non-stationary setting because of variation in the parameters. We alternatively define a \textit{surrogate parameter} $\bar{\theta}_t = W_{t,\lambda}^{-1} ( \sum_{l=1}^{t-1} \gamma^{-l} X_l X_l^T \theta_l^{\star} + \lambda \gamma^{-(t-1)} \theta_t^{\star} )$, which belongs to $\mathcal{C}_t$ with probability at least $1-\delta$, which is formally stated in Lemma \ref{lemma:wls-non}.

\subsection{RANDOMIZED EXPLORATION}
In this section, we propose two randomized algorithms for non-stationary stochastic linear bandits, Discounted randomized LinUCB (D-RandLinUCB) and Discounted Linear Thompson Sampling (D-LinTS). To gracefully adapt to environmental variation, the weighted method with exponentially discounting factor is directly applied to both RandLinUCB and Gaussian-LinTS, respectively. 
The random perturbations are injected to D-RandLinUCB and D-LinTS in different fashions: either by replacing optimism with simple randomization in deciding the confidence level or perturbing estimates before maximizing the expected rewards.

\subsubsection{Discounted Randomized Linear UCB}
Following the optimism in face of uncertainty principle, D-LinUCB \citep{russac2019weighted} chooses an action by maximizing the upper confidence bound of expected reward based on $\hat{\theta}_t^{wls}$ and confidence level $a$. Motivated by the recent work of \citet{vaswani2019old}, our first randomized algorithm in non-stationary linear bandit setting is constructed by replacing confidence level $a$ with a random variable $Z_t \sim \cD$ and this non-stationary variant of RandLinUCB algorithm is called Discounted Randomized LinUCB (D-RandLinUCB, Algorithm \ref{alg:alg2}),
\begin{align*}
    \text{D-LinUCB} : X_t & = \arg\max_{x \in \cX_t} \langle x, \hat{\theta}^{wls}_t \rangle + a \|x\|_{V_t^{-1}} \\
    {\small \text{D-RandLinUCB}} : X_t & = \arg\max_{x \in \cX_t} \langle x, \hat{\theta}^{wls}_t   \rangle + Z_t \|x\|_{V_t^{-1}}.
\end{align*}
\\[-2.5em]
\begin{algorithm}[h]
   \caption{Discounted Randomized Linear UCB}
   \label{alg:alg2}
\begin{algorithmic}
    \STATE {\bfseries Input:} $\lambda\ge 1$, $0<\delta<1$, $0< \gamma<1$, and $a>0$
    \STATE Initialize $W = \lambda I_d$, $\tilde{W} = \lambda I_d$, $\bar{b} = 0$, and $\hat{\theta}=0$.
    \FOR{$t=1$ {\bfseries to} $T$}
    \STATE Randomly sample $Z_t$ from a distribution $\mathcal{D}(\delta,a)$
    \STATE Obtain $UCB(x) = x^T \hat{\theta} + Z_t \sqrt{x^T W^{-1} \tilde{W} W^{-1} x}$
    \STATE $X_t = \arg\max_{x\in \cX_t} UCB(x)$
    \STATE Play action $X_t$ and receive reward $Y_t$
    \STATE Update $W = \gamma W + X_{t} X_{t}^T + (1-\gamma) \lambda I_d$,\\
    $\tilde{W} = \gamma^2 \tilde{W} + X_{t} X_{t}^T + (1-\gamma^2) \lambda I_d$,\\
    $\bar{b} = \gamma \bar{b} +X_t Y_t$, $\hat{\theta} = W^{-1} \bar{b}$.
    \ENDFOR
\end{algorithmic}
\end{algorithm}
\\[-2.5em]
\subsubsection{Discounted Linear Thompson Sampling}
The idea of perturbing estimates via random perturbation in LinTS algorithm can be directly applied to non-stationary setting by replacing $\hat{\theta}_t^{ls}$ and Gram matrix $V_{t,\lambda}$ with the weighted least-squares estimator $\hat{\theta}_t^{wls}$ and its corresponding matrix $V_{t} = W_{t,\lambda} \tilde{W}_{t,\lambda}^{-1} W_{t,\lambda}$.
We call it Discounted Linear Thompson Sampling (D-LinTS, Algorithm \ref{alg:alg3}). The motivation of D-LinTS arises from its equivalence to D-RandLinUCB with {\em decoupled} perturbations $Z_{x,t}$ for all $x \in \cX_t$ in round $t$ as
\\[-1em]
\begin{align*}
    \tilde{f}_t(x) & = \langle x, \tilde{\theta}^{wls}_t\rangle = \langle x, \hat{\theta}^{wls}_t\rangle + x^T W_{t,\lambda}^{-1} \tilde{W}_{t,\lambda}^{1/2} Z^{(t)} \\
    & =  \langle x, \hat{\theta}^{wls}_t\rangle + Z_{x,t} \|x\|_{V_t^{-1}}
\end{align*}
\\[-1em]
where $Z^{(t)} \sim \mathcal{N}(0_d,a^2 I_d), Z_{x,t} \sim \mathcal{N}(0,a^2)$. 
Perturbations above are decoupled in that random perturbation are not shared across every arm, and thus they obtain more variation and accordingly $(\log K)^{3/8}$ larger regret bound than that of D-RandLinUCB algorithm that is associated with {\em coupled} perturbations $Z_{t}$. By paying a logarithmic regret gap in terms of $K$ at a cost, the innate perturbation of D-LinTS allows itself to have an offline optimization oracle access in contrast to D-LinUCB and D-RandLinUCB. Therefore, D-LinTS algorithm can be efficient in computation even with an infinite action set. 
\\[-1em]
\begin{algorithm}[h]
   \caption{Discounted Linear Thompson Sampling}
   \label{alg:alg3}
\begin{algorithmic}
    \STATE {\bfseries Input:} $\lambda \ge 1$, $0< \gamma<1$, and $a>0$
    \STATE Initialize $W = \lambda I_d$, $\tilde{W} = \lambda I_d$, $\bar{b} = 0$ and $\hat{\theta} = 0$.
    \FOR{$t=1$ {\bfseries to} $T$}
    \STATE Obtain $\tilde{\theta} = \hat{\theta} + W^{-1}   \tilde{W}^{1/2} Z$, $Z \sim \mathcal{N}(0, a^2I_d)$
    \STATE {\bfseries Oracle :} $X_t = \arg\max_{x\in \cX_t} \langle x, \tilde{\theta} \rangle$
    \STATE Play action $X_t$ and receive reward $Y_t$
    \STATE Update $W = \gamma W + X_{t} X_{t}^T + (1-\gamma) \lambda I_d$,\\
    $\tilde{W} = \gamma^2 \tilde{W} + X_{t} X_{t}^T + (1-\gamma^2) \lambda I_d$, \\$\bar{b} = \gamma \bar{b} +X_t Y_t$, $\hat{\theta} = W^{-1} \bar{b}$.
    \ENDFOR
\end{algorithmic}
\end{algorithm}
\\[-2em]
\subsection{ANALYSIS}
We construct a general regret bound for linear bandit algorithm on the top of prior work of \citet{kveton2019perturbed}. The difference from their work is that an action set $\cX_t$ varies from time $t$ and can have infinite arms. Also, non-stationary environment is considered where true parameter $\theta_t^{\star}$ changes within total variation $B_T$. The expected dynamic regret is decomposed into surrogate regret and bias arising from total variation. 
\begin{align*}
   & E[R(T)]  = \sum_{t=1}^T  E[\langle x_t^{\star}- X_t, \theta^{\star}_t \rangle] \\
   & = \sum_{t=1}^T E[ \langle x_t^{\star}-X_t, \bar{\theta}_t \rangle] + \sum_{t=1}^T E[\langle x_t^{\star}-X_t, \theta^{\star}_t - \bar{\theta}_t \rangle ]\\
    &\le  \sum_{t=1}^T E[ \langle x_t^{\star}-X_t, \bar{\theta}_t \rangle] + 2  \sum_{t=1}^T \| \theta_t^{\star} - \bar{\theta}_t\|_2
\end{align*}
\subsubsection{Surrogate Instantaneous Regret}

To bound the surrogate instantaneous regret $E[\langle x_t^{\star}-X_t, \bar{\theta}_t \rangle]$, we newly define three events $E^{wls}, E_t^{conc}$, and $E_t^{anti}$:
\begin{align*}
    &E^{wls} =\{ \forall (x, t) \in \bar{\cX}_T; |\langle x, \hat{\theta}^{wls}_t - \bar{\theta}_t \rangle | \le c_1 \|x\|_{V_t^{-1}} \}, \\
    &E_t^{conc} =\{\forall x\in \cX_t; |\tilde{f}_t(x)-\langle x, \hat{\theta}^{wls}_t\rangle|  \le c_2 \|x\|_{V_t^{-1}} \}, \\
    &E_t^{anti} =\{ \tilde{f}_t(x_t^{\star})-\langle x_t^{\star}, \hat{\theta}^{wls}_t\rangle  > c_1 \|x_t^{\star}\|_{V_t^{-1}} \},
\end{align*}
where $\bar{\cX}_T = \{ (x,t) : x \in \cX_t, t \in [T] \}$. The choice of $\tilde{f}_t(x)$ is made by algorithmic design, which decides choices on both $c_1$ and $c_2$ simultaneously. In round $t$, we consider the general algorithm which maximizes perturbed expected reward $\tilde{f}_t(x)$ over action space $\cX_t$. The following theorem is a extension of Theorem 1 \citep{kveton2019perturbed} to the time-evolving environment.
\begin{theorem}\label{thm:1step-non}
Assume we have $\lambda \ge 1$ and $c_1, c_2 \ge 1$ satisfying $P(E^{wls}) \ge 1-p_1$, $P(E_t^{conc}) \ge 1-p_2$, and $P(E_t^{anti}) \ge p_3$, and $c_3 = 2d \log(\frac{1}{\gamma}) + 2\frac{d}{T} \log (1+\frac{1}{d\lambda(1-\gamma)})$. Let $A$ be an algorithm that chooses arm $X_t = \arg\max_{\cX_t} \tilde{f}_t(x)$ at time $t$. Then the expected surrogate instantaneous regret of $A$, $E[\langle x_t^{\star}-X_t, \bar{\theta}_t \rangle]$ is bounded by 
\begin{align*}
 p_2 + (c_1+ c_2) \big( 1+ \frac{2}{p_3-p_2}\big) E_t \big[\min (1,\|X_t\|_{V_t^{-1}}) \big].  
\end{align*}
\end{theorem}
\begin{proof}
Firstly, we newly define $\Delta_x = \langle x_t^{\star}-x, \bar{\theta}_t \rangle$ in round $t$. Given history $\mathcal{H}_{t-1}$, we assume that event $E^{wls}$ holds and let $\bar{S}_t = \{ x \in \cX_t : (c_1 +c_2) \|x\|_{V_t^{-1}} \ge \Delta_x \text{ and } \Delta_x \ge 0 \}$ be the set of arms that are under-sampled and worse than $x_t^{\star}$ given $\bar{\theta}_t$ in round $t$. Among them, let $U_t = \arg \min_{x \in \bar{S}_t} \|x\|_{V_t^{-1}}$ be the least uncertain under-sampled arm in round $t$. By definition of the optimal arm, $x_t^{\star} \in \bar{S}_t$. The set of sufficiently sampled arms is defined as $S_t = \{ x \in \cX_t : (c_1 +c_2) \|x\|_{V_t^{-1}} \le \Delta_x \text{ and } \Delta_x \ge 0 \}$ and let $c = c_1 + c_2$. Note that any actions $x \in \cX_t$ with $\Delta_x <0$ can be neglected since the regret induced by these actions are always negative so that it is upper bounded by zero. Given history $\mathcal{H}_{t-1}$, $U_t$ is deterministic term while $X_t$ is random because of innate randomness in $\tilde{f}_t$. Thus surrogate instantaneous regret can be bounded as, 
\begin{align*}
    &\Delta_{X_t}  = \Delta_{U_t} + \langle U_t, \bar{\theta}_t\rangle - \langle X_t, \bar{\theta}_t\rangle \\
    & \le \Delta_{U_t} + \tilde{f}_t(U_t) - \tilde{f}_t(X_t) + c \|X_t\|_{V_t^{-1}} + c \|U_t\|_{V_t^{-1}}\\
    & \le c \|X_t\|_{V_t^{-1}} + 2c \|U_t\|_{V_t^{-1}}.
\end{align*}
Thus, the expected surrogate instantaneous regret can be bounded as,
\begin{align*}
    E_t[\Delta_{X_t}] & = E_t[\Delta_{X_t} I\{E_t^{conc}\}] + E_t[\Delta_{X_t} I\{\bar{E}_t^{conc}\}] \\
    &\le c E_t [\|X_t\|_{V_t^{-1}}] + 2c \|U_t\|_{V_t^{-1}}+ P_t(\bar{E}_t^{conc}) \\
    &\le c E_t [\|X_t\|_{V_t^{-1}}] + 2c \|U_t\|_{V_t^{-1}}+ p_2\\
    & \le c E_t [\|X_t\|_{V_t^{-1}}] + 2c \frac{E_t [\|X_t\|_{V_t^{-1}}]}{P_t (X_t \in \bar{S}_t)}+ p_2\\
    & = c(1+\frac{2}{P_t (X_t \in \bar{S}_t)}) E_t [\|X_t\|_{V_t^{-1}}] + p_2\\
    & \le c(1+\frac{2}{p_3-p_2}) E_t [\|X_t\|_{V_t^{-1}}] + p_2\\
    & \le c(1+\frac{2}{p_3-p_2}) E_t [\min(1,\|X_t\|_{V_t^{-1}})] + p_2.
\end{align*}
The third inequality holds because of definition of $U_t$ that is the least uncertain in $\bar{S}_t$ and deterministic as follows,
\begin{align*}
    E_t [\|X_t\|_{V_t^{-1}}] & \ge E_t [\|X_t\|_{V_t^{-1}}| X_t \in \bar{S}_t] \cdot P_t (X_t \in \bar{S}_t) \\
    &\ge \|U_t\|_{V_t^{-1}} \cdot P_t (X_t \in \bar{S}_t).
\end{align*}
The last inequality works because $\lambda_{\min}(V_t) \ge 1$ implies $\|X_t\|_{V_t^{-1}} \le 1$.

The second last inequality holds since on event $E_t^{ls}$,
\begin{align*}
    P_t (X_t \in \bar{S}_t) & \ge P_t \big(\exists x \in \bar{S}_t : \tilde{f}_t(x) \ge \max_{y\in S_t} \tilde{f}_t(y) \big) \\
    & \ge P_t \big(\tilde{f}_t(x_t^{\star}) \ge \max_{y\in S_t} \tilde{f}_t(y) \big) \\
    & \ge P_t \big(\tilde{f}_t(x_t^{\star}) \ge \max_{y\in S_t} \tilde{f}_t(y), {E}_t^{conc}  \big)\\
    & \ge P_t \big(\tilde{f}_t(x_t^{\star}) \ge \langle x_t^{\star}, \bar{\theta}_t\rangle, {E}_t^{conc}  \big)\\
    &\ge P_t \big(\tilde{f}_t(x_t^{\star}) \ge \langle x_t^{\star }, \bar{\theta}_t\rangle ) -P_t \big(\bar{E}_t^{conc}  \big)\\
    &\ge p_3 - p_2.
\end{align*}
The fourth inequality holds since for any $y \in S_t$, $\tilde{f}_t(y) \le \langle y, \bar{\theta}_t \rangle + c\|y\|_{V_t^{-1}} \le \langle y, \bar{\theta}_t\rangle + \Delta_y = \langle x_t^{\star}, \bar{\theta}_t\rangle.$
\end{proof}
In the following three lemmas, the probability of events $E^{wls}, E_t^{conc}$, and $E_t^{anti}$ can be controlled with optimal choices of $c_1$ and $c_2$ for D-RandLinUCB and D-LinTS algorithms.
\begin{lemma}[Proposition 3, \citet{russac2019weighted}] \label{lemma:wls-non}
For $\lambda >0$, and
$c_1 = \sqrt{2 \log T + d \log (1+ \frac{1- \gamma^{2(T-1)}}{
\lambda d(1-\gamma^2)})} + \lambda^{1/2}$, the event $E^{wls}$ holds with probability at least $1-1/T$.
\end{lemma}
\begin{lemma}[Concentration] \label{lemma:nonst-conc}
Given history $\mathcal{H}_{t-1}$, \\
(a) D-RandLinUCB : $\tilde{f}_t(x) = \langle x, \hat{\theta}^{wls}_t \rangle + Z_t \cdot \| x \|_{V_t^{-1}}$ where $Z_t \sim \mathcal{N}(0,a^2)$, and $c_2= a\sqrt{2\log (T/2)}$. Then, $P(\bar{E}_t^{conc}) \le 1/T$.\\
(b) D-LinTS : $\tilde{f}_t(x) = \langle x, \hat{\theta}^{wls}_t \rangle + x^T W_{t,\lambda}^{-1} \tilde{W}_{t,\lambda}^{1/2} Z^{(t)}$, where $Z^{(t)} \sim \mathcal{N}(0,a^2 I_d)$, and $c_2 = a\sqrt{2\log (KT/2)}$. Then, $P(\bar{E}_t^{conc}) \le 1/T$.
\end{lemma}
\begin{proof}
(a) We have $\tilde{f}_t(x) = \langle x, \hat{\theta}^{wls}_t \rangle + Z_t \| x \|_{V_t}^{-1}$ in D-RandLinUCB algorithm, and thus
\begin{align*}
    &P(\bar{E}_t^{conc}) = 1- P(E_t^{conc}) \\
    &=  1 - P(\forall x\in \cX_t; |\tilde{f}_t(x)-\langle x, \hat{\theta}^{wls}_t\rangle|  \le c_2 \|x\|_{V_t^{-1}}) \\
    & = 1- P(\forall x\in \cX_t; |Z_t|\cdot \|x\|_{V_t^{-1}} \le c_2\|x\|_{V_t^{-1}})\\
    & = 1- P(|Z_t| \le c_2) \because \text{Lemma} \;\ref{lemma:gaussian}\\
    &\le 1/T, \;\text{where} \; c_2 = a\sqrt{2\log (T/2)}.
\end{align*}
(b) Given history $\mathcal{H}_{t-1}$, we have $\tilde{f}_t(x) = \langle x, \hat{\theta}^{wls}_t \rangle + x^T W_{t,\lambda}^{-1} \tilde{W}_{t,\lambda}^{1/2} Z^{(t)}$ is equivalent to $\tilde{f}_t(x) = \langle x, \hat{\theta}^{wls}_t \rangle + Z_{t,x} \cdot \| x \|_{V_t}^{-1}$ where $Z_{t,x} \sim \mathcal{N}(0,a^2)$ by the linear invariant property of Gaussian distributions. Thus,
\begin{align*}
    &P(\bar{E}_t^{conc}) = 1- P(E_t^{conc}) \\
    &=  1 - P(\forall x\in \cX_t; |\tilde{f}_t(x)-\langle x, \hat{\theta}^{wls}_t\rangle|  \le c_2 \|x\|_{V_t^{-1}}) \\
    & = 1- P(\forall x\in \cX_t; |Z_{t,x}|\cdot \|x\|_{V_t^{-1}} \le c_2\|x\|_{V_t^{-1}})\\
    & = 1- P(\forall x\in \cX_t; |Z_{t,x}| \le c_2) \because \text{Lemma} \;\ref{lemma:gaussian}\\
    &\le 1/T, \;\text{where}\; c_2 = a\sqrt{2\log (KT/2)}.
\end{align*}\\[-3em]
\end{proof}
\begin{lemma}[Anti-concentration]\label{lemma:nonst-anti}
Given $\mathcal{H}_{t-1}$, \\
(a) D-RandLinUCB : $\tilde{f}_t(x) = \langle x, \hat{\theta}^{wls}_t \rangle + Z_t \| x \|_{V_t^{-1}}$, where $Z_t \sim \mathcal{N}(0,a^2)$. Then, $P(E_t^{anti}) \ge e^{-1/4}/(8\sqrt{\pi})$ when we have $a^2 = 14c_1^2$.\\
(b) D-LinTS : $\tilde{f}_t(x) = \langle x, \hat{\theta}^{wls}_t \rangle + x^T W_{t,\lambda}^{-1} \tilde{W}_{t,\lambda}^{1/2} Z^{(t)}$ where $Z^{(t)} \sim \mathcal{N}(0,a^2 I_d)$. If we assume $a^2 = 14c_1^2$, then $P(E_t^{anti}) \ge e^{-1/4}/(8\sqrt{\pi})$ .
\end{lemma}
\begin{proof}
(a) We denote perturbed expected reward as $\tilde{f}_t(x) = \langle x, \hat{\theta}^{wls}_t \rangle + Z_t \| x \|_{V_t}^{-1}$ for D-RandLinUCB. Thus,
\begin{align*}
    P(E_t^{anti})&= P(\tilde{f}_t(x_t^{\star})-\langle x_t^{\star}, \hat{\theta}^{wls}_t\rangle  > c_1 \|x_t^{\star}\|_{V_t^{-1}})\\
    & = P(Z_t \ge c_1) \ge \exp \big( -7c_1^2/(2a^2) \big)/(8\sqrt{\pi}) \\
    &= e^{-1/4}/(8\sqrt{\pi})\quad \text{where } a^2 = 14c_1^2.
\end{align*}
(b) In the same way as the proof of Lemma \ref{lemma:nonst-conc} (b), $\tilde{f}_t(x) = \langle x, \hat{\theta}^{wls}_t \rangle + x^T W_{t,\lambda}^{-1} \tilde{W}_{t,\lambda}^{1/2} Z^{(t)}$ is equivalent to $\tilde{f}_t(x) = \langle x, \hat{\theta}^{wls}_t \rangle + Z_{t,x} \cdot \| x \|_{V_t}^{-1}$ where $Z_{t,x} \sim \mathcal{N}(0,a^2)$. Thus,
\begin{align*}
    &P(E_t^{anti})= P(\tilde{f}_t(x_t^{\star})-\langle x_t^{\star}, \hat{\theta}^{wls}_t\rangle  > c_1 \|x_t^{\star}\|_{V_t^{-1}})\\
    & = P(Z_{t,x_t^{\star}} \ge c_1) \ge \exp \big( -7c_1^2/(2a^2) \big)/(8\sqrt{\pi}) \\
    &= e^{-1/4}/(8\sqrt{\pi})\quad \text{where } a^2 = 14c_1^2.
\end{align*}\\[-3em]
\end{proof}
\subsubsection{Dynamic Regret}
The dynamic regret bound of general randomized algorithm is stated below.
\begin{theorem}[Dynamic Regret]\label{thm:regret-non}
Assume we have $c_1, c_2\ge 1$ satisfying $P(E^{wls}) \ge 1-p_1$, $P(E_t^{conc}) \ge 1-p_2$, and $P(E_t^{anti}) \ge p_3$, and $c_3 = 2d \log(\frac{1}{\gamma}) + 2\frac{d}{T} \log (1+\frac{1}{d\lambda(1-\gamma)})$. Let $A$ be an algorithm that chooses arm $X_t = \arg\max_{\cX_t} \tilde{f}_t(x)$ at time $t$. The expected dynamic regret of A is bounded as for any integer $D>0$,
\begin{align*}
&E[R(T)] \le \;(c_1+ c_2) \big( 1+ \frac{2}{p_3-p_2}\big) \sqrt{c_3 T }  \\
     & + T (p_1+p_2) + d + 2\sqrt{\frac{d}{\lambda}} D^{3/2} B_T  + \frac{4}{\lambda} \frac{\gamma^D}{1-\gamma}T.
\end{align*}
\end{theorem}
\begin{proof}
The dynamic regret bound is decomposed into two terms, $(A)$ expected surrogate regret and $(B)$ bias arising from time variation on true parameter,
\begin{align*}
    E[R(T)] &\le  \sum_{t=1}^T E[ \langle x_t^{\star}-X_t, \bar{\theta}_t \rangle] + 2  \sum_{t=1}^T \| \theta_t^{\star} - \bar{\theta}_t\|_2.
\end{align*}
The expected surrogate regret term $(A)$ is bounded by
\begin{align*}
   &
    \sum_{t=d+1}^T E [\langle x_t^{\star}-X_t, \bar{\theta}_t \rangle I\{E^{wls}\}] + T \cdot P(\bar{E}^{wls}) + d\\
    & \le (c_1+ c_2) \sum_{t=1}^T \big( 1+ \frac{2}{p_3-p_2}\big) E_t \big[\|X_t\|_{V_t^{-1}} \big] \\
    & + T(p_1+p_2) +d\\
    & \le (c_1+ c_2) \sum_{t=1}^T \big( 1+ \frac{2}{p_3-p_2}\big) E_t \big[\min (1,\|X_t\|_{V_t^{-1}}) \big] \\
    & + T(p_1+p_2) +d\\
    & \le(c_1+ c_2) \big( 1+ \frac{2}{p_3-p_2}\big) \sqrt{c_3 T } + T(p_1+p_2) +d.
\end{align*}
The first inequality holds due to Theorem \ref{thm:1step-non}. . The second inequality works because both dynamic regret and surrogate regret are upper bounded by $2T$ and $c_1 + c_2 \ge 2$. Also, the last inequality holds by Lemma \ref{lemma:norm-non} in Appendix \ref{app:reg-non}.
For any integer $D>0$, the bias term $(B)$ is bounded as
\begin{align*} 
    &(B) 
    = 2\sum_{t=1}^T \| W_{t,\lambda}^{-1} \sum_{l=1}^{t-1} \gamma^{-l}X_l X_l^T (\theta_l^{\star} - \theta_t^{\star})\|_2\\
     &\le  2\sum_{t=1}^T\| W_{t,\lambda}^{-1} \sum_{l=t-D}^{t-1} \gamma^{-l}X_l X_l^T (\theta_l^{\star} - \theta_t^{\star}) \|_2 \\
    &+ 2\sum_{t=1}^T\| W_{t,\lambda}^{-1} \sum_{l=1}^{t-D-1} \gamma^{-l}X_l X_l^T (\theta_l^{\star} - \theta_t^{\star})\|_2\\
    &\le 2\sum_{t=1}^T\sum_{m=t-D}^{t-1} \| W_{t,\lambda}^{-1} \sum_{l = t-D}^{m}\gamma^{-l} X_l X_l^T  (\theta_m^{\star} - \theta_{m+1}^{\star}) \|_2 \\
    &+\sum_{t=1}^T\frac{2}{\lambda}\|\sum_{l=1}^{t-D-1} \gamma^{t-l-1}X_l X_l^T (\theta_l^{\star} - \theta_t^{\star})\|_2\\
    &\le 2\sqrt{\frac{dD}{\lambda}} \sum_{t=1}^T \sum_{m=t-D}^{t-1}  \| \theta^{\star}_m - \theta^{\star}_{m+1}\|_2 + \frac{4}{\lambda} \frac{\gamma^D}{1-\gamma}T \\
    & \le 2\sqrt{\frac{d}{\lambda}}D^{3/2} B_T  + \frac{4}{\lambda} \frac{\gamma^D}{1-\gamma}T.
\end{align*}
The second inequality holds by interchanging the order of summations and $W_{t,\lambda}^{-2} \preccurlyeq (\frac{\gamma^{t-1}}{\lambda})^2I_d$. The second last inequality is derived from the fact that for $t-D \le m \le t-1$, $\lambda_{\max} \Big( W_{t,\lambda}^{-1} \sum_{l=t-D}^m \gamma^{-l} X_l X_l^T \Big) \le 1$.
\end{proof}
With the optimal choice of $c_1, c_2$ and $a$ derived from Lemma \ref{lemma:wls-non}-\ref{lemma:nonst-anti}, the dynamic regret bounds of D-RandLinUCB and D-LinTS are stated below.
\begin{corollary}[Dynamic Regret of D-RandLinUCB]\label{coro:non}
Suppose
\begin{align*}
    &c_1 = \sqrt{2 \log T + d \log (1+ \frac{1- \gamma^{2(T-1)}}{
\lambda d(1-\gamma^2)})} + \lambda^{1/2},\\
    &c_2 = a\sqrt{2\log (T/2)}, \text{and}\; a^2  = 14 c_1^2.
\end{align*}
Let A be D-RandLinUCB (Algorithm \ref{alg:alg2}). If $B_T$ is known, then with optimal choice of $D = {\log T}/(1-\gamma)$ and $\gamma = 1- d^{-\frac{1}{4}}B_T^{\frac{1}{2}}T^{-\frac{1}{2}}$, the expected dynamic regret of A is asymptotically upper bounded by $\mathcal{O} ( d^{\frac{7}{8}} B_T^{\frac{1}{4}} T^{\frac{3}{4}})$ as $T \rightarrow \infty$.

If $B_T$ is unknown, D-RandLinUCB together with Bandits-over-Bandits mechanism enjoys the expected dynamic regret of $\mathcal{O} ( d^{\frac{7}{8}} B_T^{\frac{1}{4}} T^{\frac{3}{4}})$.
\end{corollary}
\begin{corollary}[Dynamic Regret of D-LinTS]\label{coro:non2}
Suppose
\begin{align*}
    &c_1  = \sqrt{2 \log T + d \log (1+ \frac{1- \gamma^{2(T-1)}}{
\lambda d(1-\gamma^2)})} + \lambda^{1/2},\\
    &c_2 = a\sqrt{2\log (KT/2)}, \text{and}\; a^2 = 14 c_1^2
\end{align*}
Let A be D-LinTS (Algorithm \ref{alg:alg3}). If $B_T$ is known, then with optimal choice of $D = {\log T}/(1-\gamma)$ and $\gamma = 1- d^{-\frac{1}{4}}(\log K)^{-\frac{1}{4}}B_T^{\frac{1}{2}}T^{-\frac{1}{2}}$, the expected dynamic regret of A is asymptotically upper bounded by $\mathcal{O} ( d^{\frac{7}{8}}(\log K)^{\frac{3}{8}} B_T^{\frac{1}{4}} T^{\frac{3}{4}})$ as $T \rightarrow \infty$.

If $B_T$ is unknown, D-LinTS together with Bandits-over-Bandits mechanism enjoys the expected dynamic regret of $\mathcal{O} ( d^{\frac{7}{8}} (\log K)^{\frac{3}{8}} B_T^{\frac{1}{4}} T^{\frac{3}{4}})$.
\end{corollary}
The detailed proof of Theorem \ref{thm:regret-non} and Corollary \ref{coro:non} and \ref{coro:non2} are deferred to Appendix \ref{app:reg-non}. The details for the case of unknown $B_T$ are deferred to Appendix \ref{app:bob}.

Note that exponentially discounting weights can be replaced by sliding window strategy or restarted strategy to accommodate to evolving environment. We can construct sliding-window randomized LinUCB (SW-RandLinUCB) and sliding-window linear Thompson sampling (SW-LinTS), or restarting randomized LinUCB (Restart-RandLinUCB) and restarting linear Thompson sampling (Restart-LinTS) via two perturbation approaches, and they maintain the trade-off between oracle efficiency and theoretical guarantee. With unknown total variation $B_T$, we can also utilize Bandits-over-Bandits mechanism by applying the EXP3 algorithm over these algorithms with different window sizes \citep{cheung2019hedging} or epoch sizes \citep{zhao2020simple, zhao2021non}, respectively.

{\bf Trade-off between Oracle Efficiency and Minimax Optimality :}
Corollary \ref{coro:non} shows that D-RandLinUCB does not match the lower bound for dynamic regret, $\Omega (d ^{2/3}B_T^{1/3} T^{2/3})$, but it achieve the same dynamic regret bound as that of three non-randomized algorithms such as SW-LinUCB, D-LinUCB and Restart-LinUCB. 
However, D-RandLinUCB is computationally inefficient as D-LinUCB in large action space since the spectral norm of each action in terms of matrix $V_t^{-1}$ should be computed in every round $t$. In contrast, D-LinTS algorithm relies on offline optimization oracle access via perturbation and thus can be efficiently implemented in infinite-arm setting, and even contextual bandit setting. As a cost of its oracle efficiency, D-LinTS achieves the dynamic regret bound $(\log K)^{3/8}$ worse than that of D-RandLinUCB in finite-arm setting. There exist two variations in D-LinTS; algorithmic variation generated by perturbing an estimate $\hat{\theta}_t^{wls}$ and environmental variation induced by time-varying environments. Two variations are hard to distinguish from the learner's perspective, and thus the effect of algorithmic variation is alleviated by being partially absorbed in environmental variation. This is why D-LinTS and D-LinUCB produce $d^{3/8}$ gap of dynamic regret bounds with infinite set of arms which is less than $d^{1/2}$ gap between regret bounds of LinUCB and LinTS in the stationary environment.

\section{NUMERICAL EXPERIMENT}\label{sec:exp}
\begin{figure*}[t]
\centering
\begin{subfigure}{.33\textwidth}
  \centering
  \includegraphics[width=\linewidth]{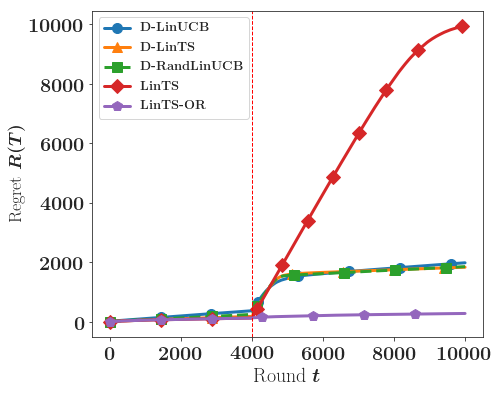}
  \caption{$(d,K) =(10, 10)$}
  \label{fig:fig2a}
\end{subfigure}
\begin{subfigure}{.33\textwidth}
  \centering
  \includegraphics[width=\linewidth]{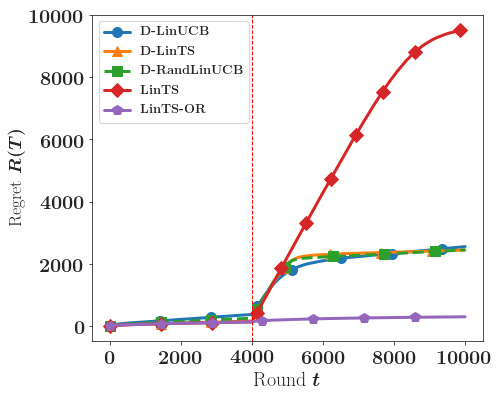}
  \caption{$(d,K) =(20, 10)$}
  \label{fig:fig2c}
\end{subfigure}
\begin{subfigure}{.33\textwidth}
  \centering
  \includegraphics[width=\linewidth]{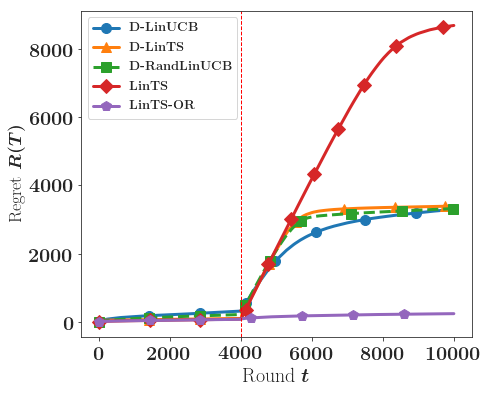}
  \caption{$(d,K) =(50, 10)$}
  \label{fig:fig2e}
\end{subfigure}%
\\
\begin{subfigure}{.33\textwidth}
  \centering
  \includegraphics[width=\linewidth]{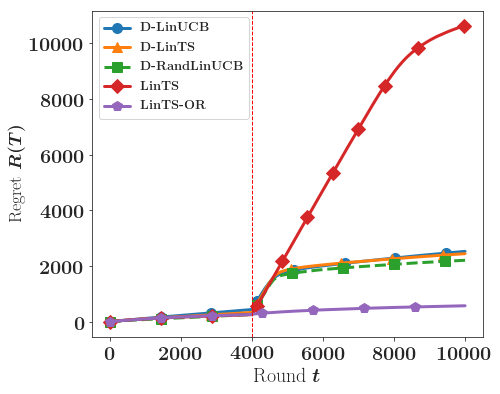}
  \caption{$(d,K) =(10, 100)$}
  \label{fig:fig2b}
\end{subfigure}
\begin{subfigure}{.33\textwidth}
  \centering
  \includegraphics[width=\linewidth]{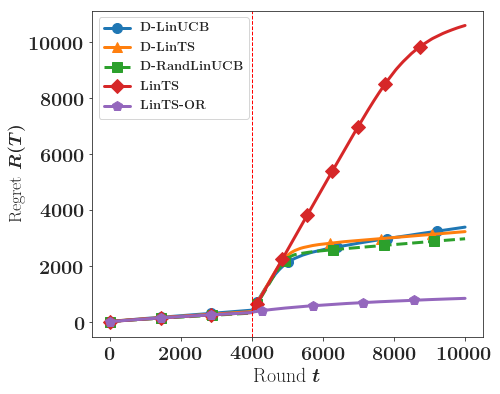}
  \caption{$(d,K) =(20, 100)$}
  \label{fig:fig2d}
\end{subfigure}%
\begin{subfigure}{.33\textwidth}
  \centering
  \includegraphics[width=\linewidth]{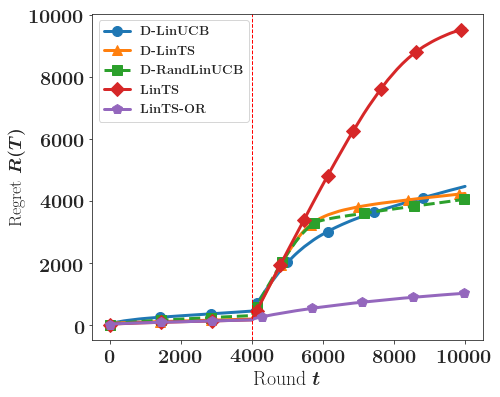}
  \caption{$(d,K) =(50, 100)$}
  \label{fig:fig2f}
\end{subfigure}
\caption{Plots of cumulative dynamic regret for $d = 10, 20$, and $50$ and $K=10$, and $100$.}
\label{fig:fig}
\end{figure*}
In simulation studies\footnote{\url{https://github.com/baekjin-kim/NonstationaryLB}}, we evaluate the empirical performance of D-RandLinUCB and D-LinTS. We use a sample of 30 days of Criteo live traffic data \citep{diemert2017attribution} by 10\% downsampling without replacement. Each line corresponds to one impression that was displayed to a user with contextual variables as well as information of whether it was clicked or not. We kept \textit{campaign} variable and categorical variables from \textit{cat1} to \textit{cat9} except for \textit{cat7}. We experiment with several dimensions $d = 10, 20, 50$ and the number of arms $K= 10, 100$. Among all one-hot coded contextual variables, $d$ feature variables were selected by Singular Value Decomposition for dimensionality reduction. We construct two linear models and the model switch occurs at time 4000. The parameter $\theta^{\star}$ in the initial model is obtained from linear regression model and we obtain true parameter $\theta^{\star}$ in the second model by switching the signs of 60\% of the components of $\theta^{\star}$. In each round, $K$ arms given to all algorithms are equally sampled from two separate pools of 10000 arms corresponding to clicked or not clicked impressions. The rewards are generated from linear model with additional Gaussian noise of variance $\sigma^2 =0.15$.

We compare randomized algorithms D-RandLinUCB and D-LinTS to discounted linear UCB (D-LinUCB) as a benchmark. Also, we compare them to linear Thompson sampling (LinTS) and oracle restart LinTS (LinTS-OR). An oracle restart knows about the change-point and restarts the algorithm immediately after the change. In D-RandLinUCB, we use truncated normal distribution with zero mean and standard deviation $2/5$ over $[0, \infty)$ as $\mathcal{D}$ to ensure that its randomly chosen confidence bound belongs to that of D-LinUCB with high probability. Also, we use non-inflated version by setting $a=1$ when implementing both LinTS and D-LinTS \citep{vaswani2019old}. The regularization parameter is $\lambda = 1$, the time horizon is $T = 10000$ and the cumulative dynamic regret of algorithms are averaged over 100 independent replications in Figure \ref{fig:fig}.

We observe the following patterns in Figure \ref{fig:fig}. First, two randomized algorithms, D-RandLinUCB and D-LinTS outperform the non-randomized one, D-LinUCB when action space is quite large ($K=100$) in figure \ref{fig:fig2b}, \ref{fig:fig2d}, and \ref{fig:fig2f}. In the setting where the number of arms is small ($K=10)$, however, non-randomized algorithm (D-LinUCB) performs better than two randomized algorithms once relatively high-dimension feature is considered (figure \ref{fig:fig2c} and \ref{fig:fig2e}), while three nonstationary algorithms show almost similar performance when feature is low-dimensional (figure \ref{fig:fig2a}).

Second, D-RandLinUCB always works better than D-LinTS in all scenarios. Though D-LinTS can enjoy oracle efficiency in computational aspect, it has slightly worse regret bound than D-RandLunUCB. The difference in theoretical guarantees can be empirically evaluated in this result. The poor performance of D-LinUCB in large action space is due to its very large confidence bound so that the issue regarding conservatism can be partially tackled by randomizing a confidence level in D-RandLinUCB.

Lastly, the interesting observation in figure \ref{fig:fig2f}, non-randomized algorithm D-LinUCB shows better performance in recovering a reliable estimator after experiencing a change point than other two competitors in the initial phase. It takes longer time for randomized algorithms to recover their performance. This is because the agent cannot distinguish which factor causes this nonstationarity it is experiencing: either randomness inherited from algorithm nature or environmental change. However, randomized algorithms eventually beat the non-randomized competitor in the final phase.

\section{CONCLUSION}
For non-stationary linear bandits, we propose two randomized algorithms, Discounted Randomized LinUCB and Discounted Linear Thompson Sampling which are the first of their kind by replacing optimism with a simple randomization in UCB-type algorithms, or by adding the random perturbations to estimates, respectively. We analyzed their dynamic regret bounds and evaluated their empirical performance in a simulation study.

The existence of a randomized algorithm that enjoys both theoretical optimality and oracle efficiency is still open in stationary and non-stationary stochastic linear bandits. 
\vspace{-1pt}
\section*{ACKNOWLEDGEMENT}
\vspace{-1pt}
We acknowledge the support of NSF CAREER grant IIS-1452099 and the UM-LSA Associate Professor Support Fund.
\newpage
\section*{References}
\bibliography{paper}

\onecolumn
\appendix

\section{PROOF : NON-STATIONARY SETTING} \label{app:non}
\subsection{LEMMA}
\begin{lemma}[Concentration and Anti-Concentration of Gaussian distribution \citep{abramowitz1964handbook}]\label{lemma:gaussian}
Let $Z$ be the Gaussian random variable with mean $\mu$ and variance $\sigma^2$. For any $z>0$, 
\begin{align*}
\frac{1}{4\sqrt{\pi}}\exp(-\frac{7z^2}{2})  \le P(|Z-\mu|> z \sigma ) \le \frac{1}{2} \exp(-\frac{z^2}{2}).
\end{align*}
\end{lemma}
\subsection{PROOF OF THEOREM \ref{thm:regret-non}}\label{app:reg-non}
\begin{proof}[Proof of Theorem \ref{thm:regret-non}]
The dynamic regret bound is decomposed into two terms, $(A)$ expected surrogate regret and $(B)$ bias arising from variation on true parameter.
\begin{align*}
    E[R(T)] &= \sum_{t=1}^T E[ \langle x_t^{\star} - X_t, \theta^{\star}_t \rangle] = \sum_{t=1}^T E[ \langle x_t^{\star}-X_t, \bar{\theta}_t \rangle] + \sum_{t=1}^T E[\langle x_t^{\star}-X_t, \theta^{\star}_t - \bar{\theta}_t \rangle ]\\
    &\le  \sum_{t=1}^T E[ \langle x_t^{\star}-X_t, \bar{\theta}_t \rangle] + 2  \sum_{t=1}^T \| \theta_t^{\star} - \bar{\theta}_t\|_2 = (A) + (B)
\end{align*}
The expected surrogate regret term $(A)$ is bounded as,
\begin{align*}
   (A) & = \sum_{t=1}^T E[ \langle x_t^{\star}-X_t, \bar{\theta}_t \rangle] \le \sum_{t=d+1}^T E[ \langle x_t^{\star}-X_t, \bar{\theta}_t \rangle] + d \\
    &\le \sum_{t=d+1}^T E [\langle x_t^{\star}-X_t, \bar{\theta}_t \rangle I\{E^{wls}\}] + T \cdot P(\bar{E}^{wls}) + d \\
    & \le \sum_{t=d+1}^T E [\langle x_t^{\star}-X_t, \bar{\theta}_t \rangle I\{E^{wls}\}] + T p_1 + d \\
    & \le (c_1+ c_2) \big( 1+ \frac{2}{p_3-p_2}\big) E_t \big[\sum_{t=d+1}^T \min(1, \|X_t\|_{V_t^{-1}}) \big]+ T(p_1+p_2) +d \quad \because \text{Theorem } \ref{thm:1step-non} \\
    & \le(c_1+ c_2) \big( 1+ \frac{2}{p_3-p_2}\big) \sqrt{c_3 T } + T(p_1+p_2) +d \quad \because \text{Cauchy-Schwarz inequality} \;\&\; \text{Lemma } \ref{lemma:norm-non}
\end{align*}
\begin{lemma}[Corollary 4, \citet{russac2019weighted}]\label{lemma:norm-non}
For any $\lambda >0$, 
\begin{align*}
    \sum_{t=d+1}^T \min(1, \|X_t\|_{V_{t}^{-1}}^2) \le c_3 T
\end{align*}
where $c_3  = 2d \log(1/\gamma) + 2\frac{d}{T} \log (1+\frac{1}{d\lambda(1-\gamma)} )$.
\end{lemma}
The bias term $(B)$ is bounded in terms of total variation, $B_T$. We first bound the individual bias term at time $t$. For any integer $D>0$,
\begin{align*} 
&\| \theta_t^{\star} - \bar{\theta}_t\|_2 = \| W_{t,\lambda}^{-1} \sum_{l=1}^{t-1} \gamma^{-l}X_l X_l^T (\theta_l^{\star} - \theta_t^{\star})\|_2 \\ 
    &  \le  \| W_{t,\lambda}^{-1} \sum_{l=t-D}^{t-1} \gamma^{-l}X_l X_l^T (\theta_l^{\star} - \theta_t^{\star}) \|_2 + \| W_{t,\lambda}^{-1} \sum_{l=1}^{t-D-1} \gamma^{-l}X_l X_l^T (\theta_l^{\star} - \theta_t^{\star})\|_2 \\
    &  \le  \| W_{t,\lambda}^{-1} \sum_{l=t-D}^{t-1} \gamma^{-l}X_l X_l^T \sum_{m = l}^{t-1} (\theta_m^{\star} - \theta_{m+1}^{\star}) \|_2 + \|\sum_{l=1}^{t-D-1} \gamma^{-l}X_l X_l^T (\theta_l^{\star} - \theta_t^{\star})\|_{W_{t,\lambda}^{-2}} 
\end{align*}
\begin{align*}
    &  \le  \| W_{t,\lambda}^{-1} \sum_{m=t-D}^{t-1} \sum_{l = t-D}^{m} \gamma^{-l}X_l X_l^T  (\theta_m^{\star} - \theta_{m+1}^{\star}) \|_2 + \frac{1}{\lambda}\|\sum_{l=1}^{t-D-1} \gamma^{t-l-1}X_l X_l^T (\theta_l^{\star} - \theta_t^{\star})\|_2 \\
    & \le \sum_{m=t-D}^{t-1} \| W_{t,\lambda}^{-1} \sum_{l = t-D}^{m}\gamma^{-l} X_l X_l^T  (\theta_m^{\star} - \theta_{m+1}^{\star}) \|_2 + \frac{2}{\lambda} \frac{\gamma^D}{1-\gamma} \\
    & \le \sum_{m=t-D}^{t-1} \lambda_{\max} \Big( W_{t,\lambda}^{-1} \sum_{l=t-D}^m \gamma^{-l} X_l X_l^T \Big) \| \theta^{\star}_m - \theta^{\star}_{m+1}\|_2 + \frac{2}{\lambda} \frac{\gamma^D}{1-\gamma}  \\
    & \le \sqrt{\frac{dD}{\lambda}} \sum_{m=t-D}^{t-1}  \| \theta^{\star}_m - \theta^{\star}_{m+1}\|_2 + \frac{2}{\lambda} \frac{\gamma^D}{1-\gamma} 
\end{align*}
The third inequality holds due to $W_{t,\lambda}^{-2} \preccurlyeq (\frac{\gamma^{t-1}}{\lambda})^2I_d$. The last inequality works due to $\lambda_{\max} \Big( W_{t,\lambda}^{-1} \sum_{l=t-D}^m \gamma^{-l} X_l X_l^T \Big) \le 1$ for $t-D \le m \le t-1$. By combining individual bias terms over $T$ rounds, we can derive the upper bound of bias term $(B)$ as,
\begin{align*} 
(B) & = 2\sum_{t=1}^T \| \theta_t^{\star} - \bar{\theta}_t\|_2 \\
& \le 2\sum_{t=1}^T \sqrt{\frac{dD}{\lambda}} \sum_{m=t-D}^{t-1}  \| \theta^{\star}_m - \theta^{\star}_{m+1}\|_2 + \frac{4}{\lambda} \frac{\gamma^D}{1-\gamma}T\\
&\le 2\sqrt{\frac{d}{\lambda}}D^{3/2} B_T  + \frac{4}{\lambda} \frac{\gamma^D}{1-\gamma}T
\end{align*}

\begin{lemma} \label{lemma:eigenmax}
For $t-D \le m \le t-1$,
$$
\lambda_{\max}\Big( W_{t,\lambda}^{-1} \sum_{l=t-D}^m \gamma^{-l} X_l X_l^T \Big) \le \sqrt{\frac{dD}{\lambda}}.
$$ 
\end{lemma}
\begin{proof}
Denote by $\mathbb{B}(1) = \{x | \|x\|_2 =1 \}$ the unit ball.
\begin{align*}
    \lambda_{\max}\Big( W_{t,\lambda}^{-1} \sum_{l=t-D}^m \gamma^{-l} X_l X_l^T \Big) &= \sup_{z \in \mathbb{B}(1)} \Big| z^T W_{t,\lambda}^{-1} \big( \sum_{l=t-D}^m \gamma^{-l} X_l X_l^T \big) z \Big| \\
    & = \Big| z_{\star}^T W_{t,\lambda}^{-1} \big( \sum_{l=t-D}^m \gamma^{-l} X_l X_l^T \big) z_{\star} \Big| \quad z_{\star} : \text{optimizer} \\
    & \le \|z_{\star}\|_{W_{t,\lambda}^{-1}} \Big\| \sum_{l=t-D}^m \gamma^{-l} X_l X_l^T z_{\star} \Big\|_{W_{t,\lambda}^{-1}} \\
    & \le \|z_{\star}\|_{W_{t,\lambda}^{-1}} \Big\| \sum_{l=t-D}^m  \gamma^{-l} X_l \|X_l\|_2 \|z_{\star}\|_2 \Big\|_{W_{t,\lambda}^{-1}} \\
    & \le \frac{ \gamma^{(t-1)/2}}{\sqrt{\lambda}} \Big\| \sum_{l=t-D}^m \gamma^{-l} X_l \Big\|_{W_{t,\lambda}^{-1}} 
    \le \frac{\gamma^{(t-1)/2}}{\sqrt{\lambda}} \sum_{l=t-D}^m  \Big\| \gamma^{-l} X_l \Big\|_{W_{t,\lambda}^{-1}} \\
    & \le \sqrt{\frac{D}{\lambda}} \sqrt{\gamma^{t-1} \sum_{l=t-D}^m \| \gamma^{-l} X_l \|^2_{W_{t,\lambda}^{-1}}} 
     \le \sqrt{\frac{dD}{\lambda}}
\end{align*}
In this proof, we utilized the fact that for any $x$, we have $\|x\|_{W_{t,\lambda}^{-1}} \le \|x\|_2/\sqrt{\lambda \gamma^{-(t-1)}} = \frac{\|x\|_2 \gamma^{(t-1)/2}}{\sqrt{\lambda}}$. The last step makes use of the following result: for any $m \in \{t-D, \cdots, t-1\}$,
\begin{align*}
    & \gamma^{t-1}\sum_{l=t-D}^m \| \gamma^{-l} X_l \|^2_{W_{t,\lambda}^{-1}} \\
    &= \sum_{l=t-D}^m \text{tr}( \gamma^{-l} X_l^T W_{t,\lambda}^{-1} X_l) \\
    & = \text{tr}\Big( W_{t,\lambda}^{-1} \sum_{l=t-D}^m \gamma^{-l} X_l X_l^T\Big) \\
    & \le \text{tr}\Big( W_{t,\lambda}^{-1} \sum_{l=t-D}^m \gamma^{-l} X_l X_l^T\Big) + \sum_{s=m+1}^{t-1} \gamma^{-l} X_l^T W_{t,\lambda}^{-1} X_l + \lambda\gamma^{-t} \sum_{i=1}^d e_i^T W_{t,\lambda}^{-1} e_i \\
    & = \text{tr}\Big( W_{t,\lambda}^{-1} \sum_{l=t-D}^m \gamma^{-l} X_l X_l^T\Big) + \text{tr}\Big( W_{t,\lambda}^{-1} \sum_{l=m+1}^{t-1} \gamma^{-l} X_l X_l^T\Big) +
    \text{tr}\Big( W_{t,\lambda}^{-1} \lambda\gamma^{-t} \sum_{i=1}^d e_i e_i^T \Big) \\
    & = \text{tr}(I_d) = d
\end{align*}
\end{proof}

Therefore, the expected dynamic regret is bounded as,
\begin{align*}
    E[R(T)] &\le (A) + (B) \\
    & \le (c_1+ c_2) \big( 1+ \frac{2}{p_3-p_2}\big) \sqrt{c_3 T } + T(p_1+p_2) +d + 2\sqrt{\frac{d}{\lambda}}D^{3/2} B_T  + \frac{4}{\lambda} \frac{\gamma^D}{1-\gamma}T
\end{align*}

In Corollary \ref{coro:non}, the choices of $a, c_1, c_2$, and $c_3$ are
\begin{align*}
    & a^2 = 14c_1^2, \;c_1 = \sqrt{2 \log T + d \log (1+ \frac{1- \gamma^{2(T-1)}}{
\lambda d(1-\gamma^2)})} + \lambda^{1/2}\\
    &c_2 = a\sqrt{2\log (T/2)},  \text{and} \; c_3 = 2d \log(1/\gamma) + 2\frac{d}{T} \log (1+\frac{1}{d\lambda(1-\gamma)} ).
\end{align*}
 With optimal choice of 
 $$
 D = \frac{\log T}{1-\gamma},\; \gamma = 1- d^{-\frac{1}{4}}B_T^{\frac{1}{2}}T^{-\frac{1}{2}},
 $$ 
 the dynamic regret of the D-RandLinUCB algorithm is asymptotically upper bounded by $\mathcal{O} ( d^{\frac{7}{8}} B_T^{\frac{1}{4}} T^{\frac{3}{4}})$ as $T \rightarrow \infty$ as $T \rightarrow \infty$. 

In Corollary \ref{coro:non2}, the choices of $a, c_1, c_2$, and $c_3$ are
\begin{align*}
    & a^2 = 14c_1^2, \;c_1 = \sqrt{2 \log T + d \log (1+ \frac{1- \gamma^{2(T-1)}}{
\lambda d(1-\gamma^2)})} + \lambda^{1/2}\\
    &c_2 = a\sqrt{2\log (KT/2)}, \text{and} \; c_3 = 2d \log(1/\gamma) + 2\frac{d}{T} \log (1+\frac{1}{d\lambda(1-\gamma)}).
\end{align*}
 With optimal choice of 
 $$
 D = \frac{\log T}{1-\gamma},\; \gamma = 1- d^{-\frac{1}{4}}(\log K)^{-\frac{1}{4}}B_T^{\frac{1}{2}}T^{-\frac{1}{2}},
 $$ 
 the dynamic regret of the D-LinTS algorithm is asymptotically upper bounded by \\
 $\mathcal{O} ( d^{\frac{7}{8}}(\log K)^{\frac{3}{8}} B_T^{\frac{1}{4}} T^{\frac{3}{4}})$ as $T \rightarrow \infty$. 
\end{proof}

\section{Adapting to unknown non-stationarity}\label{app:bob}
The optimal discounting factor $\gamma^{\star}$ requires prior information of non-stationarity measure $B_T$, which is unavailable in general. We make up for the lack of this information via running the EXP3 algorithm as a meta algorithm to adaptively choose the optimal discounting factor. This method of adapting to unknown non-stationarity is called as Bandits-over-Bandits (BOB) \citep{cheung2019hedging}.

The BOB mechanism divides the entire time horizon into $[T/H]$ blocks of equal length $H$ rounds, and specifies a set $J \subset [H]$ from which each critical window size $D_i$ is drawn from. For each block $i$, the BOB mechanism selects a critical window size $D_i$ and starts a new copy of D-RandLinUCB algorithm. On top of this, the BOB mechanism separately maintains the EXP3 algorithm to carefully control the selection of critical window size for each block, and the total reward of each block is used as bandit feedback for the EXP3 algorithm.

We set $H = d^{\frac{1}{4}} T^{\frac{1}{2}}$ and 
 we consider D-RandLinUCB algorithm together with BOB mechanism. The details for D-LinTS algorithm will be skipped since its dynamic regret bound can be obtained in a very similar fashion. With choices of parameters, 
\begin{align*}
    & a^2 = 14c_1^2, \;c_1 = \sqrt{2 \log T + d \log (1+ \frac{1- \gamma^{2(T-1)}}{
\lambda d(1-\gamma^2)})} + \lambda^{1/2}, \\
&c_2 = a\sqrt{2\log (T/2)}, \text{and} \;  c_3 = 2d \log(1/\gamma) + 2\frac{d}{T} \log (1+\frac{1}{d\lambda(1-\gamma)})
\end{align*}
the expected dynamic regret bound is bounded as
\begin{align}
    E[R(T)] & \le (c_1+ c_2) \big( 1+ \frac{2}{p_3-p_2}\big) \sqrt{c_3 T } + T(p_1+p_2) +d + 2\sqrt{\frac{d}{\lambda}}D^{3/2} B_T  + \frac{4}{\lambda} \frac{\gamma^D}{1-\gamma}T \\
    & \le \tilde{\mathcal{O}} \big( \sqrt{d}D^{3/2} B_T + \frac{dT}{\sqrt{D}} \big),  \label{eq:window}
\end{align}
where $D$ is called {\em critical window size} in the sense that the observations outside this critical window size would not affect the order of regret bound (only by constant instead). This quantity is closely related to discounting factor $\gamma$ in the following equation, $D = (\log T)/(1-\gamma)$. That is, to find the optimal discounting factor is equivalent to finding the optimal critical window size.

\begin{align*}
    \E [Regret_T(BOB)] =& \E [ \sum_{t=1}^T \langle x_t^{\star}, \theta_t \rangle - \sum_{t=1}^T \langle X_t, \theta_t \rangle] \\
    = & \underbrace{\E \big[ \sum_{t=1}^T \langle x_t^{\star}, \theta_t \rangle - \sum_{i=1}^{[T/H]} \sum_{t=(i-1)H+1}^{iH \wedge T} \langle X_t(D_{\dagger}), \theta_t \rangle \big]}_{(a)} \\
    & + \underbrace{\E \big[ \sum_{i=1}^{[T/H]} \sum_{t=(i-1)H+1}^{iH \wedge T} \langle X_t(D_{\dagger}) - X_t(D_i), \theta_t \rangle \big]}_{(b)} 
\end{align*}
where $D_{\dagger}$ is the best critical window size to approximate the optimal critical window size $D^{\star}$ in the pool $J = \{ H^0, [H^{\frac{1}{\Delta}}], [H^{\frac{2}{\Delta}}],\cdots, H\}$ for some positive integer $\Delta$, and we can set $H = [d^{\frac{1}{4}}T^{\frac{1}{2}}]$ and $\Delta = [\log H]$. Recall that $D^{\star} =  d^{\frac{1}{4}}B_T^{-\frac{1}{2}}T^{\frac{1}{2}}\log T$. It suffices to bound terms (a) and (b).

The term (a) is bounded using Equation \ref{eq:window},
\begin{align*}
    (a) & = \E \big[\sum_{i=1}^{[T/H]} \sum_{t=(i-1)H+1}^{iH \wedge T} \langle  x_t^{\star} - X_t(D_{\dagger}), \theta_t \rangle \big]\\
    & = \sum_{i=1}^{[T/H]} \tilde{\mathcal{O}}\big( \sqrt{d}D_{\dagger}^{3/2} B_T(i) + \frac{dH}{\sqrt{D_{\dagger}}} \big) \\
    & = \tilde{\mathcal{O}} \big( \sqrt{d}D_{\dagger}^{3/2} B_T + \frac{dT}{\sqrt{D_{\dagger}}} \big)
\end{align*}
where $B_T(i) = \sum_{t=(i-1)H+1}^{iH \wedge T-1} \|\theta_t - \theta_{t+1}\|$ is the total variation in block $i$.

Next, we bound the term (b) as below. The number of rounds in a block is $[T/H]$ and the number of possible options of $D_i$ is $|J| = \Delta +1$.
\begin{align*}
    (b) & \le \tilde{\mathcal{O}}(\sqrt{H |J| T})
\end{align*}
where this inequality follows by the same argument as in the sliding window based approach \citep{cheung2019learning}.

Combining term (a) and (b), the regret of the BOB mechanism is
$$
\E [Regret_T(BOB)] = \tilde{\mathcal{O}} \big( \sqrt{d}D_{\dagger}{}^{3/2} B_T + \frac{dT}{\sqrt{D_{\dagger}}}+ \sqrt{H |J| T} \big).
$$

\paragraph{Case 1 : $D^{\star} \le H$.} This condition implies that there exists $\epsilon_1$ and $\epsilon_2$ such that $B_T = \tilde{\Omega}(d^{\epsilon_1} T^{\epsilon_2})$ where at least one of $\epsilon_1$ and $\epsilon_2$ is positive, and thus $D_{\dagger}$ can automatically approximate to the nearly optimal critical window size $D^{\star}$. Then, the dynamic regret of the BOB mechanism becomes
\begin{align*}
    \E [Regret_T(BOB)] & = \tilde{\mathcal{O}} \big( \sqrt{d}D_{\dagger}{}^{3/2} B_T + \frac{dT}{\sqrt{D_{\dagger}}}+ \sqrt{H |J| T} \big) \\
    & = \tilde{\mathcal{O}} \big( \sqrt{d}D^{\star}{}^{3/2}H^{\frac{1}{\Delta}} B_T + \frac{dT}{\sqrt{D^{\star}H^{-\frac{1}{\Delta}}}}+ d^{\frac{1}{8}}T^{\frac{3}{4}}\Delta^{\frac{1}{2}} \big) \\
    & = \tilde{\mathcal{O}} \big( d^{\frac{7}{8}}B_T^{\frac{1}{4}}T^{\frac{3}{4}}+ d^{\frac{1}{8}}T^{\frac{3}{4}}\Delta^{\frac{1}{2}} \big) \\
    & = \tilde{\mathcal{O}} \big( d^{\frac{7}{8}}B_T^{\frac{1}{4}}T^{\frac{3}{4}} \big).
\end{align*}

\paragraph{Case 2 : $D^{\star} > H$.} This condition implies that $B_T = \tilde{\mathcal{O}}(1)$. Under this situation, $D_{\dagger}$ equals to $H$, which is the critical window size closest to $D^{\star}$, then 
the dynamic regret of the BOB mechanism becomes

\begin{align*}
    \E [Regret_T(BOB)] & = \tilde{\mathcal{O}} \big( \sqrt{d}D_{\dagger}{}^{3/2} B_T + \frac{dT}{\sqrt{D_{\dagger}}}+ \sqrt{H |J| T} \big) \\
    & = \tilde{\mathcal{O}} \big( \sqrt{d}H^{3/2} B_T + \frac{dT}{\sqrt{H}}+ \sqrt{H |J| T} \big)\\
   & = \tilde{\mathcal{O}} \big( d^{\frac{7}{8}}B_T T^{\frac{3}{4}}+ d^{\frac{7}{8}} T^{\frac{3}{4}} + d^{\frac{1}{8}}T^{\frac{3}{4}}\Delta^{\frac{1}{2}} \big) \\
    & = \tilde{\mathcal{O}} \big( d^{\frac{7}{8}}B_T^{\frac{1}{4}}T^{\frac{3}{4}} \big).
\end{align*}
The last inequality holds by the condition $B_T = \tilde{\mathcal{O}}(1)$.

\end{document}